\documentclass[runningheads]{llncs}
\usepackage[utf8]{inputenc}
\usepackage{microtype}
\usepackage{listings}
\usepackage{srcltx}

\newcommand{\E}{\textnormal{\textrm{E}}}

\newcommand\myparagraph[1]{\smallskip\noindent\textbf{#1.}}
\newcommand\mysection[1]{\vspace{-1em}\section{#1}\vspace{-0.5em}}

\let\accentvec\vec

\let\vec\accentvec

\usepackage{amsmath}
\usepackage{amssymb}
\usepackage{amsfonts}
\usepackage{verbatim}
\usepackage[textsize=tiny]{todonotes}

\title{Differentially Private Sketches for Jaccard Similarity Estimation}

\usepackage{url}

\author{
  Martin Aumüller\orcidID{0000-0002-7212-6476}
  \and 
  Anders Bourgeat
  \and 
  Jana Schmurr
}
\institute{IT University of Copenhagen\\ Copenhagen, Denmark\\ \email{\{maau,anfh,jansc\}@itu.dk}}

\begin{document}

\maketitle

\begin{abstract}
    This paper describes two locally-differential private algorithms for releasing user vectors such that the Jaccard similarity between these vectors can be efficiently estimated. 
    The basic building block is the well known MinHash method. 
    To achieve a privacy-utility trade-off, MinHash is extended in two ways using variants of Generalized Randomized Response and the Laplace Mechanism. 
    A theoretical analysis provides bounds on the absolute error  
    and experiments show the utility-privacy trade-off on synthetic and real-world data. The paper ends with a critical discussion of related work.
\end{abstract}

\mysection{Introduction}

Privacy of user data is becoming an ever increasing need 
for organizations and users alike. Multiple large-scale privacy breaches in the last years
showed how critical and vulnerable most of today's infrastructure is~\cite{dwork2014algorithmic}. 
In particular,
there is dispute about the concept of a \emph{trusted data 
curator} to whom users send their original data, and who uses this data to build models for different tasks such as targeted advertisement.
As Kearns and Roth put it in their recent book about ethical algorithms~\cite{KearnsRoth19}, ``[to] make sure that the effect of these models respect the societal norms that we want to maintain, we need to learn how to design these goals directly into our algorithms.'' In pursue of this goal,  the present paper studies how we can implicitly incorporate privacy into a similarity search system.

The concept of differential privacy as introduced by Dwork et al. in~\cite{DworkMNS06} 
defines privacy in a precise mathematical way that often allows the design 
of efficient randomized algorithms. In the case of an untrusted 
data curator, the concept can be extended to \emph{local differential privacy}, where users themselves run randomized algorithms to 
make their data private before sending it to an untrusted curator.

This paper proposes two randomized mechanisms when users
have a collection of items and are interested in finding their 
similarity with other users under the \emph{Jaccard similarity} in a private manner. The proposed  
algorithms build upon the papers~\cite{Kenthapadi13,Dhaliwal19,yanlocally} 
and a precise account of the relation will be given in the related work section at the end of this paper. In a nutshell, each user starts by applying MinHash as introduced by Broder~\cite{bro97b} with the range compression of Li and K\"onig~\cite{LiK10} (Section~\ref{sec:minhash}) to produce a sketch of their data. It is well known that these sketches can be used to efficiently estimate the original Jaccard similarity. Now, each user applies a local randomization to their sketch to satisfy the notion of differential privacy  as introduced in the next section.  One randomization mechanism is based on the concept of randomized response (Section~\ref{sec:grr}), 
the other mechanism uses the concept of Laplacian noise (Section~\ref{sec:noise:addition}).  
We provide probabilistic bounds on the estimation error of these mechanisms as Theorem~\ref{thm:utility:rr:minhash} and Theorem~\ref{thm:utility:noisy:minhash}. 
A running example of our setting and the mechanisms is provided in Appendix~\ref{app:example}.
The mechanisms will be evaluated in a real-world setting in Section~\ref{sec:experiments}. 
There we will see that they allow for precise similarity estimations if user vectors do not contain too few elements.

We hope that the proposed methods will help in building privacy-preserving similarity search systems with good utility and 
precise privacy guarantees.  

\mysection{(Local) Differential Privacy}
\label{sec:dp}

Differential privacy conveys a precise mathematically definition 
of privacy. It says that a randomized algorithm is private if 
for two ``neighboring'' databases, there must be 
a ``good enough'' probability that the algorithm produces the same output.
Here, a clean definition of \emph{neighboring} is a key criterion and we will 
introduce our notion in the next section.
While differential privacy usually works with a trusted data curator, the 
notion of \emph{local differential privacy} describes the setting in which  the user apply the randomized
algorithm themselves. Thus, the curator never sees the original data. 

\begin{definition}[Sect.~12.1~\cite{dwork2014algorithmic}]
    \label{def:ldp}
    Let $\varepsilon, \delta \geq 0$. Let $\mathcal{A}$ be a randomized algorithm with 
    output space $\mathcal{R}$. $\mathcal{A}$ satisfies
    $(\varepsilon, \delta)$-local differential privacy ($(\varepsilon, \delta)$-LDP), if 
    and only if for any neighboring input $x$ and $y$ we have:
        $\forall v \in \mathcal{R} \colon \Pr[\mathcal{A}(x) = v] \leq e^\varepsilon \Pr[\mathcal{A}(y) = v] + \delta.$
\end{definition}
\noindent Note that $\mathcal{A}$ is run by each individual user.




\mysection{Basic Setup}

Let  $\mathcal{U}$ be a collection of $n$ users and $\mathcal{I}$ be a collection of $m$ items. 
Each user has a subset of the items.
Formally, user $u \in \mathcal{U}$ is associated with 
a bit vector $x_u= (X_1, \ldots, X_m) \in \{0,1\}^m$, where $X_i = 1$ means that item $i$ is 
present in the user's item set. From a practical point of view, such a representation
is often obtained from a real-valued vector $(X'_1,\ldots,X'_m) \in \mathbb{R}^m$ by setting
$X_i = 1$ iff $X'_i \geq t$, for some chosen threshold $t \in \mathbb{R}$. 

This paper will focus on the similarity of user's item sets with regard to their \emph{Jaccard similarity}.
For two vectors $x, y \in \{0,1\}^m$, the Jaccard similarity $J(x,y) = |x \cap y|/|x \cup y|$ 
is the fraction of positions with a common one over the number of positions with at least a one.

We want to release the matrix $M = (x_u)_{u\in \mathcal{U}} \in \{0,1\}^{n\times m}$ in a
locally differential private way. This means that each user locally produces a 
differential private version $\hat{x}_u$ of $x_u$ such that if two vectors $x_u$ and
$y_u$ do not differ by much, there is a good chance that they map to the same output. 
Sending all $\hat{x}_u$ to an untrusted curator, 
we obtain a matrix $\hat{M} = (\hat{x}_u)_{u \in \mathcal{U}}$ that can be
published. 
The utility of this mapping $M \mapsto \hat{M}$ is the ability to recover from any 
two vectors $\hat{x}$ and $\hat{y}$ their original similarity $J(x, y)$.
Since the mapping introduced random noise to preserve privacy, custom similarity 
estimation algorithms are required to solve this task.

\myparagraph{Neighboring Notion}
Throughout this paper, we will often make the assumption that each user vector has at least $\tau \geq 1$ items, i.e., at least $\tau$ bits are set. 
We say that two vectors $x$ and $y$ in $\{0,1\}^m$ are neighboring if they differ in at most $\alpha$ positions. In this case, $J(x,y) \geq 1 - \alpha/\tau$.

\myparagraph{Basic Building Blocks of Differential Privacy}
We review the Laplace mechanism~\cite[Chapter 3.3]{dwork2014algorithmic} to produce differential privacy mechanisms in our context. 
The $\ell_1$-sensitivity $\Delta(f)$ of a function $f\colon\{0,1\}^m \rightarrow \mathbb{R}^K$ is defined as 
$
    \Delta(f) = \max \| f(x) - f(y) \|_1
$, where the maximum is taken over all neighboring bitstrings $x,y$.
Given $f(x) \in \mathbb{R}^K$ and a privacy budget $\varepsilon$, the Laplace mechanism returns the value $f(x) + (Y_1, \ldots, Y_K)$, where each $Y_i$ is drawn independently
from the Laplace distribution with shape parameter $\Delta(f) / \varepsilon$ and mean 0. 
\begin{theorem}[Thm 3.6~\cite{dwork2014algorithmic}]
    \label{thm:laplace}
    The Laplace mechanism preserves $(\varepsilon, 0)$-LDP.
\end{theorem}
Another way of preserving $(\varepsilon, 0)$-LDP is via generalized randomized response~\cite{wang2017locally}. The variant used in this paper will be described in Lemma~\ref{lem:grr:on:budget}.


%


\vspace{-1em}
\subsection{Jaccard Similarity Estimation via MinHash}
\label{sec:minhash}

\myparagraph{MinHash} Our approach relies on the \emph{MinHash algorithm} that was first described by Broder in~\cite{bro97b}.  
Choosing a MinHash function $h\colon \{0, 1\}^m \rightarrow [m] := \{1, \ldots, m\}$ 
amounts to choosing a random permutation $\pi$ over $[m]$. 
The hash value of $x \in \{0, 1\}^m$ is the position of the first $1$ in $x$ under $\pi$. 
MinHash has the property that for any pair $x, y \in \{0,1\}^m$, 
we have $\Pr[h(x) = h(y)] = J(x,y)$ where the probability is taken over the random choice of $h$. 
Repeating this construction $K$ times results in an output $(h_1(x), \ldots, h_K(x)) \in [m]^K$.
By linearity of expectation, the value $\frac{1}{K} \sum_{i = 1}^K [h_i(x) = h_i(y)]$ is an unbiased estimator of $J(x,y)$.


\myparagraph{$b$-bit MinHash} 
Li and König described in~\cite{LiK10} the following twist to the
standard MinHash approach. 
For an integer $B \geq 2$, choosing a \emph{range-$B$ MinHash} function amounts to choosing a MinHash function
$h_\text{min}\colon \{0,1\}^m \rightarrow [m]$ 
and a universal hash function~\cite{carter1979universal} $h_\text{uni}\colon [m] \rightarrow [B]$. 
The range-$B$ MinHash function
is $h := h_\text{uni} \circ h_\text{min} \colon \{0,1\}^m \rightarrow [B]$.
This mapping has the property that $\Pr[h(x) = h(y)] = (1 - J(x,y)) 1/B + J(x,y)$, since with probability $J(x,y)$ the MinHash value is identical---which yields a collision---and with probability $1-J(x,y)$ the MinHash value is different but the random mapping generates a collision.
\cite{LiK10} discussed the case $B = 2^b$ for $b \geq 1$ where the
hash function gives a $b$-bit value. 
In this paper we will use their approach for general $B \geq 2$.

\vspace{-1em}
\subsection{Generalized Randomized Response for Close Vectors}
To have a chance for good utility of our mechanisms, 
we will use the additive $\delta$ summand available in LDP (cf.~Def~\ref{def:ldp}) to 
collect cases where the mapping $h$ maps two neighboring user vectors 
far away from each other. We then provide $\varepsilon$-LDP on the remaining cases. We will need two technical lemmata.

\begin{lemma}
    \label{lem:differences}
    Let $x, y \in \{0, 1\}^m$ such that $J(x, y) \geq 1 - \alpha/\tau$. Let $\delta > 0$.
    Let $h_1, \ldots, h_K$ be a collection of $K$ random range-$B$ MinHash functions. 
    Let $x^\ast = (h_1(x), \ldots, h_K(x))$ and $y^\ast = (h_1(y), \ldots, h_K(y))$.
    With probability at least $ 1 - \delta$, the number of positions where $x^\ast$ and $y^\ast$ differ is at most 
    $K (\alpha/\tau) \left(1 - \frac{1}{B}\right) + \sqrt{3 \ln (1/\delta)\left(1 - \frac{1}{B}\right)  K\alpha/\tau }$.
\end{lemma}

\begin{proof}
    For each $i \in [K]$, define the random variable $X_i = [h_i(x) \neq h_i(y)]$. 
    Let $X = \sum_{i = 1}^K X_i$ denote the number of differences between $x^\ast$ and $y^\ast$. 
    Since all $X_i$ are independent and $\Pr(X_i = 1) = (1-J(x,y)) \left(1 - \frac{1}{B}\right) \leq \alpha/\tau \left(1 - \frac{1}{B}\right)$, 
    we have $\E[X] \leq  K\alpha/\tau \left(1 - \frac{1}{B}\right)$. 
    Using the Chernoff bound $\Pr(X > (1 + \beta) E[X]) \leq \exp\left( -\beta^2/3 \E[X]\right)$ \cite[Theorem 1.1]{dp09} 
    with $\beta = \sqrt{3 \ln (1/\delta)/E[X]}$ proves the lemma.
\end{proof}
The next lemma shows that we can avoid loosing a factor $K$ in the privacy budget\footnote{Traditionally, a standard application of the composition theorem~\cite{dwork2014algorithmic} shows that the composition of $K$ $\varepsilon$-DP mechanisms satisfies $(K\varepsilon)$-DP.} when using generalized randomized response~\cite{wang2017locally} on vectors with few differences. 

\begin{lemma}
    \label{lem:grr:on:budget}
    Fix $\varepsilon > 0$. Let $x, y \in [B]^K$ be two arbitrary vectors that differ in at most $L$ positions.
    Let $\varepsilon' = \varepsilon / L$. 
    Let $\mathcal{A}$ be generalized randomized response mapping from $z 
    \in [B]^K$ to $z^\ast \in [B]^K$ 
    such that with probability $e^{\varepsilon'} / (e^{\varepsilon'} + B - 1)$ we have that $z_i^\ast = z_i$, and 
    otherwise $z_i^\ast$ is uniformly picked from $[B] - \{z_i\}$. Then $\mathcal{A}$ is 
    $\varepsilon$-differentially private.
\end{lemma}

\begin{proof}
    Fix an arbitrary $v \in [B]^K$. 
    We have to show that $\frac{\Pr[A(x) = v]}{\Pr[A(y) = v]} \leq e^\varepsilon$.
    Let the set $I_{x,v}$ collect all positions in which $x_i = v_i$, and let $N_{x,v}$ collect all positions in which
    $x_i \neq v_i$.
    We observe that
    $    \Pr[A(x) = v] = \prod_{i \in I_{x,v}} \frac{e^{\varepsilon'}}{e^{\varepsilon'} + B - 1} \cdot \prod_{i \in N_{x,v}} \frac{1}{{e^{\varepsilon'} + B - 1}}.
    $
    The expression for $\Pr[A(y) = v]$ follows analogously. 
    Let $D = \{i \mid x_i \neq y_i\}$ denote all positions where $x$ and $y$ differ. Because 
    all terms where $x$ and $y$ are identical cancel out, we may conclude that
    \begin{align*}
        \frac{\Pr[A(x) = v]}{\Pr[A(y) = v]} {=} \frac{\prod_{i \in I_{x,v} \cap D} \frac{e^{\varepsilon'}}{e^{\varepsilon'} + B - 1}  \prod_{i \in N_{x,v}\cap D} \frac{1}{{e^{\varepsilon'} + B - 1}}}{\prod_{i \in I_{y,v}\cap D} \frac{e^{\varepsilon'}}{e^{\varepsilon'} + B - 1} \prod_{i \in N_{y,v}\cap D} \frac{1}{{e^{\varepsilon'} + B - 1}}}
        {\leq}\prod_{i \in D}\frac{\frac{e^{\varepsilon'}}{e^{\varepsilon'} + B - 1}}{\frac{1}{e^{\varepsilon'} + B - 1}} \leq e^{\varepsilon' L} {=} e^\varepsilon. 
    \end{align*}
\end{proof}

\mysection{LDP Sketches via Generalized Randomized Response}
\label{sec:grr}

This section introduces an $(\varepsilon, \delta)$-locally differential private algorithm to produce the user vectors $\hat{x}_u$ using generalized randomized response.

The idea of the following algorithm is that each user receives the description of 
$K$ range-$B$ MinHash functions that map from $[m]$ to $[B]$ for $B \geq 2$.
Each user applies the range-$B$ MinHash functions and perturbs the hash value using 
a variant of generalized randomized response~\cite{wang2017locally}. We proceed to describe the \texttt{RRMinHash} approach. An example is given in Figure~\ref{fig:scheme} in Appendix~\ref{app:example}, top row.

\myparagraph{Preprocessing}
Each user accesses $K \geq 1$ range-$B$ MinHash functions $h_1,\ldots, h_K$ shared among all users.
Each user $u$ applies $h_1, \ldots, h_K$ to their vector $x_u$ to obtain $x^\ast_u \in [B]^K$. 
Now, each position of $x^\ast_u$ is perturbed using generalized randomized response (Lemma~\ref{lem:grr:on:budget}) with an 
individual privacy budget of $\varepsilon' = \varepsilon / L$ to generate the response $\hat{x}_u$, where $L$ is an upper bound on 
the number of differences between neighboring user vectors as in Lemma~\ref{lem:differences}. 
$\hat{x}_u$ is the public response of user $u$.

\begin{lemma}
    The randomized mechanism $x \mapsto \hat{x}$ is $(\varepsilon, \delta)$-LDP.
\end{lemma}

\begin{proof}
    Fix $\varepsilon, \delta > 0$ and let $x,y \in \{0, 1\}^m$ such that they differ in at 
    most $\alpha$ positions. By Lemma~\ref{lem:differences}, with probability at least $1-\delta$,
    the vectors $x^\ast$ and $y^\ast$ differ in at most $L = \lceil K (\alpha/\tau) \left(1 - \frac{1}{B}\right) + \sqrt{3 \ln (1/\delta)\left(1 - \frac{1}{B}\right)  K\alpha/\tau }\rceil$ positions. 
    If $x^\ast$ and $y^\ast$ differ in at most $L$ positions, Lemma~\ref{lem:grr:on:budget}  
    guarantees that the mapping $x^\ast \mapsto \hat{x}$ is $\varepsilon$-differential private. 
\end{proof}

\myparagraph{Similarity Estimation} Given two responses $\hat{x} \in [B]^K$ and $\hat{y} \in [B]^K$, count collisions to obtain 
$p_\text{col} = \sum [\hat{x}_i = \hat{y}_i]/ K$. Given $p_\text{col}$, $B$, and $p^\ast = e^{\varepsilon'}/(e^{\varepsilon'} + B - 1)$, we estimate the Jaccard similarity 
of $x$ and $y$ as
\begin{equation}
    \hat{J}_{RR}(\hat{x}, \hat{y}) = 
        \frac{(B-1) (B \cdot p_\text{col} - 1)}{(B \cdot p^\ast - 1)^2}\label{eq:grr:estimate}
\end{equation}

\begin{lemma}
    \label{lem:unbiased:estimator}
    $\hat{J}_{RR}(\hat{x}, \hat{y})$ is an unbiased estimator of  $J(x,y)$.
\end{lemma}

\begin{proof}
    We  proceed in two parts. 
    First, we  calculate the probability of the event ``$\hat{x}_i$ = $\hat{y}_i$''. 
    Next, we  connect this probability to the estimation given above.

    To compute the collision probability, we split up the probability space in two stages. 
    In the first stage, we  condition on the events ``$x^\ast_i = y^\ast_i$'' and ``$x^\ast_i \neq y^\ast_i$'', 
    i.e. on whether the range-$B$ MinHash values collide or not.
    In the second stage, we  calculate the probability that the perturbed responses collide.
    As discussed in Section~\ref{sec:minhash}, a random range-$B$ MinHash function has the 
    property that $\Pr[x^\ast_i = y^\ast_i] = \frac{(B-1) J(x,y) + 1}{B}$.

    Given that $x^\ast_i = y^\ast_i$, we observe $\hat{x}_i = \hat{y}_i$ if both keep their answer,
    or if both change their answer to the same of the other $B - 1$ possible responses. Since both
    pick a choice uniformly at random, this means that 
        $\Pr[\hat{x}_i = \hat{y}_i \mid x^\ast_i = y^\ast_i] = (p^\ast)^2 + \frac{(1-p^\ast)^2}{B-1}$.
    Consider that the event $x^\ast_i \neq y^\ast_i$ happened. In this case, we observe a collision 
    of the perturbed values in the following cases: (i) one response is truthful, the other is changed and 
    picks the truthful response as answer, and (ii) both responses are obtained by changing the answer,
    and they both choose the same answer at random. Computing these probabilities, we conclude that
        $\Pr[\hat{x}_i = \hat{y}_i \mid x^\ast_i \neq y^\ast_i] = 2 p^\ast \left(1-p^\ast\right) \frac{1}{B-1} + (1-p^\ast)^2 \left(1 - \frac{1}{B-1}\right)^2 \frac{1}{B-2}.$ 
    The last term is obtained by first conditioning that neither choice picks the other's truthful answer, and then using the random choice of the remaining $B - 2$ buckets. 

    Putting everything together, we obtain
    \begin{align*}
        &\Pr[\hat{x}_i = \hat{y}_i] = \frac{(B-1)J(x,y) + 1}{B}\left((p^\ast)^2 + \frac{(1-p^\ast)^2}{B-1}\right) \\&+ \left(1{-}\frac{(B - 1)J(x,y) + 1}{B}\right)\left(2 p^\ast \left(1{-}p^\ast\right) \frac{1}{B{-}1} {+} (1{-}p^\ast)^2 \left(1 {-} \frac{1}{B{-}1}\right)^2 \frac{1}{B{-}2}\right).
    \end{align*}
    Simplifying this formula by collecting terms yields
    \begin{equation}
        \Pr[\hat{x}_i = \hat{y}_i] = \frac{J(x,y) + B J(x,y) p^\ast(B p^\ast - 2) + B - 1}{B (B - 1)}.
        \label{eq:coll:prob:rr}
    \end{equation}
    Solving \eqref{eq:coll:prob:rr} for $J(x,y)$ and 
    using linearity of expectation to connect $p_\text{col}$ to $\Pr[\hat{x}_i = \hat{y}_i]$ results in \eqref{eq:grr:estimate}.
\end{proof}


\myparagraph{Utility Analysis} Next we will discuss probabilistic bounds on the absolute error that the similarity estimation algorithm achieves on the private vectors.
This means that we want upper bound the value $\vert \hat{J}_{RR}(\hat{x}, \hat{y}) - J(x,y)\vert$. 
In the following, we will consider the absolute error in the case $\hat{J}_{RR}(\hat{x}, \hat{y}) > J(x,y)$. The case $\hat{J}_{RR}(\hat{x}, \hat{y}) < J(x,y)$ follows by symmetry.

\begin{lemma}
    \label{lem:load:bound:1}
With probability at least $1 - \delta$,  
\begin{equation}
\vert \hat{J}_{RR}(\hat{x}, \hat{y}) - J(x,y) \vert \leq \sqrt{\frac{3 \ln(1/\delta) B^3 (1 + p^\ast(Bp^\ast -2))}{K (Bp^\ast - 1)^4}}.
\end{equation}
\end{lemma}

\begin{proof}
    Fix $x$ and $y$. 
    We let $X_i$ be the indicator variable for the event ``$\hat{x}_i = \hat{y}_i$''.
    Define $X = X_1 + \cdots + X_K$. 
    By~\eqref{eq:coll:prob:rr} we know that $X_i$ is Bernoulli-distributed with
        $q:= \Pr[X_i = 1] = \frac{J(x,y) + B J(x,y) p^\ast(B p^\ast - 2) + B - 1}{B (B - 1)}.$
    Again using a Chernoff bound, we see that with probability at least $1 - \delta$,
    $X \leq E[X] + \sqrt{E[X] 3 \ln (1/\delta)}$. Assume from here on that this 
    inequality holds.  From \eqref{eq:grr:estimate}, we start by observing that
    \begin{align*}
        &\hat{J}_{RR}(\hat{x}, \hat{y}) {=}  \frac{(B-1) (B \cdot X/K {-} 1)}{(B \cdot p^\ast - 1)^2}
        \leq \frac{(B-1) (B (\E[X] + \sqrt{\E[X] 3\ln(1/\delta)})/K - 1)}{(B \cdot p^\ast - 1)^2}\\
        &\stackrel{\text{Lem~\ref{lem:unbiased:estimator}}}{=} J(x,y) {+} \frac{(B-1)B \sqrt{\E[X]3 \ln (1/\delta)}}{K(Bp^\ast - 1)^2}
        {<} J(x,y) {+} \sqrt{\frac{3 \ln(1/\delta) B^3 (1 {+} p^\ast(Bp^\ast {-}2))}{K (Bp^\ast - 1)^4}}.
    \end{align*}
\end{proof}

\begin{theorem} Fix $\varepsilon, \delta_{\text{DP}}, \delta_{\text{fail}} > 0$.
    There exists $B$ and $K$ such that with probability at least $1 - \delta_{\text{fail}}$,$
        \vert \hat{J}_{RR}(\hat{x}, \hat{y}) - J(x,y) \vert = O(\sqrt{\alpha/(\tau \cdot \varepsilon}))$.
    The constant hidden in the big-Oh notation depends on $\delta_{\text{DP}}$ and $\delta_{\text{fail}}$.
    \label{thm:utility:rr:minhash}
\end{theorem}

\begin{proof}
Lemma~\ref{lem:load:bound:1} tells us that for every choice of $B$ and $K$, 
with probability at least $1 - \delta_{\text{fail}}$ it holds that
\begin{equation}
\label{eq:load:1}
    \vert \hat{J}_{RR}(\hat{x}, \hat{y}) - J(x,y) \vert \leq \sqrt{\frac{3 \ln(1/\delta) B^3 (1 + p^\ast(Bp^\ast -2))}{K (Bp^\ast - 1)^4}}.
\end{equation}

Since $p^\ast \leq 1$, we continue to bound the right-hand side of \eqref{eq:load:1} by 
$    \sqrt{\frac{3 \ln (1/\delta_{\text{fail}})}{K}} \cdot B^2/(Bp^\ast -1)^2.$
Assume that $Bp^\ast \geq 2$, which means that $\varepsilon' \geq \ln(2(B-1)/(B-2))$ and 
$\varepsilon \geq L \cdot \varepsilon'$.
Since $1/(x-1) \leq 2/x$ for $x \geq 2$, we continue to bound the absolute error from above by
    $\sqrt{\frac{3 \ln (1/\delta_{\text{fail}})}{K}} \cdot B^2(2/(Bp^\ast))^2 \leq \sqrt{\frac{48 \ln (1/\delta_{\text{fail}})}{K}} (1 + (B - 1)/e^{\varepsilon'})^2.$
Now, we may set $B = 3$ since it makes the numerator as small as possible. 
($B = 2$ is no valid choice because of the assumption $B p^\ast \geq 2$.)
Choosing $K = \Theta(\tau \varepsilon/\alpha)$,
the absolute error is bounded by $\Theta(\sqrt{\alpha/(\tau \varepsilon)})$ for all $\varepsilon \geq L \ln 4$.
\end{proof}

\mysection{LDP Sketches via the Laplace Mechanism}
\label{sec:noise:addition}

This section introduces an $(\varepsilon, \delta)$-LDP protocol for generating private user vectors $\hat{x}_u$ using the Laplace Mechanism.
As before, for fixed integers $B$ and $K$, we use $K$ range-$B$ MinHash functions such that each user produces a sketch in $[B]^K$.

Let $x$ and $y$ be neighboring vectors and let $x^\ast$ and $y^\ast$ be the two sketches in $[B]^K$. As before, with probability at least $1 - \delta$, we can assume that $x^\ast$ and $y^\ast$ differ in at most $L = K (\alpha/\tau) \left(1 - \frac{1}{B}\right) + \sqrt{3 \ln (1/\delta)\left(1 - \frac{1}{B}\right)  K\alpha/\tau}$ positions. 

Before we can use Theorem~\ref{thm:laplace}, we have to compute the sensitivity $\Delta$ of the local sketches under the assumption that neighboring vectors differ in at most $L$ positions. Since each coordinate in which the two vectors differ contributes at most $B - 1$ to the $\ell_1$ norm, the sensitivity is at most $\Delta := L (B - 1)$. 
According to~Theorem~\ref{thm:laplace}, adding Laplace noise with scale $\Delta/\varepsilon$ to $x^\ast$ to produce $\hat{x}$ guarantees $(\varepsilon, 0)$-differential privacy as long as the number of differences is at most $L$. With probability at most $\delta_\text{DP}$, there are more than $L$ differences.  
We are now ready to describe the \texttt{NoisyMinHash} approach, see the example in Figure~\ref{fig:scheme}, bottom row.

\myparagraph{Preprocessing} Let $K, B, \alpha,$ and $\tau$ be integers, and let $\varepsilon > 0$ and $\delta > 0$ be the privacy budget. Choose $K$ range-$B$ MinHash functions $h_1, \ldots, h_K$ and distribute them to the users. 
Each user with vector $x$ returns $\hat{x} = (h_1(x) + N_{x,1}, \ldots, h_K(x) + N_{x, K}) \in \mathbb{R}^K$, 
where each $N_{x,i} \sim \text{Lap}(\Delta/\varepsilon)$ with 
$\Delta =  (B-1) \left(K (\alpha/\tau) \left(1 - \frac{1}{B}\right) + \sqrt{3 \ln (1/\delta)\left(1 - \frac{1}{B}\right)  K\alpha/\tau}\right)$.

\myparagraph{Similarity Estimation} Given $\hat{x}$ and $\hat{y}$ from $\mathbb{R}^K$, return 
\begin{equation}
    \label{eq:estimate}
    \hat{J}_{Lap}(\hat{x}, \hat{y}) = \frac{(B^2 - 1) K - 6\sum_{i = 1}^K (\hat{x}_i - \hat{y}_i)^2 + 24 K (\Delta/\varepsilon)^2}{(B-1)(B+1) K}.
\end{equation}
Notably, the estimation algorithm just computes the squared Euclidean distance and adjusts for the noise added. 

\begin{lemma}
    \label{lem:laplace:unbiased}
    $\hat{J}_{Lap}(\hat{x}, \hat{y})$ is an unbiased estimator for $J(x, y)$. 
\end{lemma}

\begin{proof}
    Given $x$ and $y$ from $\{0, 1\}^m$, apply \texttt{NoisyMinHash} to compute $\hat{x}, \hat{y} \in \mathbb{R}^K$. 
    Using linearity of expectation, we proceed as follows:
    \begin{align*}
        &\E\left[\sum_{i = 1}^K (\hat{x}_i - \hat{y}_i)^2\right] = \sum_{i = 1}^K E[(\hat{x}_i - \hat{y}_i)^2] = K  \E[((h_1(x) - h_1(y)) + (N_{x, 1} - N_{y,1}))^2]\\
                &= K  \E[(h_1(x) {-} h_1(y))^ 2 {+} 2 (h_1(x) {-} h_1(y)) (N_{x, i} {-} N_{y, i}) {+} (N_{x, i} - N_{y, i})^2]\\
                &\stackrel{(1)}{=} K  (\E[(h_1(x) {-} h_1(y))^ 2] + \E[(N_{x, i} - N_{y, i})^2])\\
                &\stackrel{(2)}{=} K  \E[(h_1(x) {-} h_1(y))^ 2] + 2K \text{Var}[N_{x,i}]
                =K  \E[(h_1(x) {-} h_1(y))^ 2] + 4K (\Delta/\varepsilon)^2.
    \end{align*}
    In our calculations, both (1) and (2) used that $N_{x,i}$ and $N_{y, i}$ are independently chosen,  
    $\E[N_{x,i}] = \E[N_{y,i}] = 0$, and Var$[N_{y,i}] = 2(\Delta/\varepsilon)^2$. 

    Let $x^\ast = h_1(x)$ and $y^\ast = h_1(y)$. We continue by calculating $\E[(x^\ast - y^\ast)^ 2]$ as
    \begin{align*}
        &\sum_{j = 1}^{B - 1} j^2 \Pr[\vert x^\ast - y^\ast \vert = j]
        = \sum_{j = 1}^{B - 1} j^2 \frac{(B{-}1) ( 1 {-} J(x,y))}{B} \Pr[\vert x^\ast {-} y^\ast \vert {=} j \mid x^\ast \neq y^\ast]\\
        &\stackrel{(1)}{=} \sum_{j = 1}^{B - 1} j^2 \frac{(B-1) ( 1 - J(x,y))}{B} 2(B - j)/B
        = \frac{2(B{-}1) (1 {-} J(x,y))}{B^2}\sum_{j = 1}^{B - 1} j^2 (B - j)\\
        &=  \frac{2(B-1) (1 - J(x,y))}{B^2} \cdot  \frac{B^2}{12} (B+1) (B - 1)
        = \frac{(B-1)^2 (B+1)(1-J(x,y))}{6},
    \end{align*}
    where (1) is obtained by noticing that for a fixed $x^\ast$ (with $B$ choices), there are $B-2j$ choices with two $y^\ast$ such that $\vert x^\ast - y^\ast\vert = j$, and there are $2j$ choices with only one choice $y^\ast$.  
    Putting everything together, we summarize
    \begin{equation*}
        \E\left[\sum_{i = 1}^K (\hat{x}_i - \hat{y}_i)^2\right] = \frac{K(B-1)^2 (B+1)(1-J(x,y))}{6} +  4K (\Delta/\varepsilon)^2.
    \end{equation*}
    The result is obtained by rearranging terms.
\end{proof}

\myparagraph{Utility Analysis} 

\begin{theorem} Fix $\varepsilon, \delta_{\text{DP}}, \delta_{\text{fail}} > 0$.
    There exists $B$ and $k$ such that with probability at least $1 - \delta_{\text{fail}}$, $
        \vert \hat{J}_{Lap}(\hat{x}, \hat{y}) - J(x,y) \vert = \tilde{O}\left((\alpha/\tau)^{4/5} \cdot \varepsilon^{-2/5}\right)$.
    The constant hidden in the big-Oh notation depends on $\delta_{\text{DP}}$ and $\delta_{\text{fail}}$, and the tilde notation suppresses
    polylogarithmic factors.
    \label{thm:utility:noisy:minhash}
\end{theorem}

\begin{proof}
    We first describe and analyze the two events which constitute the failure probability $\delta_\text{fail}$.
    Next we proceed to analyze the estimation error under the condition that none of these events occur.
    We will only analyze the case that $\hat{J}_{Lap}(\hat{x}, \hat{y})$ is larger than $J(x,y)$. 
    The other case follows by symmetry.

    First, we assume that the number of differences between $x^\ast$ and $y^\ast$ (among the $K$ functions) does not 
    differ by more than a value $L'$ from its expectation. 
    This is true for $L' = \sqrt{3 \ln (2/\delta_\text{fail})\left(1 - \frac{1}{B}\right)  K\alpha/\tau} $ by Lemma~\ref{lem:differences} for failure probability $\delta_\text{fail} / 2$. 
    Second, we use Theorem~3.8 in~\cite{dwork2014algorithmic} (reproduced in Appendix~\ref{app:laplace:utility}) that says that with probability at least
    $1 - \delta_\text{fail} / 4$, the maximum absolute difference in a coordinate of $\hat{x}$ compared 
    to $x^\ast$ is at most $D = \ln(4K/\delta_\text{fail}) \Delta/\varepsilon$. 

    By a union bound, with probability at least $1-\delta_\text{fail}$ none of these events occur, i.e., we observe a deviation of at most $L'$ in the number of differences of two vectors $x^\ast$ and $y^\ast$ from their expectation, 
    and the Laplace noise added to both $x^\ast$ and $y^\ast$
    keeps all coordinates within $D$ in their absolute value.
    Under this condition, we will study the value $|X - \E[X]|$ for the random variable $X = \sum_{i = 1}^K (\hat{x}_i - \hat{y}_i)^2$.
    If this value is at most $t$, the absolute estimation error is at most
    $\frac{6t}{(B-1)(B+1)K}$, cf. \eqref{eq:estimate}.

    As in the proof of Lemma~\ref{lem:laplace:unbiased}, we split up $\hat{x}_i$ into $x^\ast_i$ and $N_i$ to calculate
        $\sum_{i = 1}^K (\hat{x}_i - \hat{y}_i)^2 = \sum_{i = 1}^K \left((x^\ast_i - y^\ast_i)^2 + 2 (x^\ast_i - y^\ast_i) (N_i - N_j) + (N_i - N_j)^2\right).$
    By our second condition, we may assume that $\vert N_i - N_j \vert \leq 2D$, which means that the last summand is at most $4KD^2$ over the whole sum.
    For the first summand, we use the first condition that says that the number of observed differences is within $L'$ from its expectation. 
    Since each individual term in the sum contributes
    at most $(B-1)^2$, the deviation from the expectation over the whole sum is not more than $(B-1)^2 L'$. 
    Lastly, using both conditions, the contribution of the middle term over the whole sum is bounded by $ 4D (K\alpha/\tau (1-1/B) + L')$.

    Using $P := K\alpha/\tau (1-1/B)$ and rewriting $L' = \sqrt{3 P \ln(2/\delta_\text{fail})}$, we can put the observations from above together and conclude that with probability at least $1-\delta_\text{fail}$ the estimation error is
    \[
        O\left(\frac{(B-1)^2\sqrt{3P \ln(2/\delta_\text{fail})} + KD^2 + D (P + \sqrt{3P \ln(2/\delta_\text{fail})})}{(B-1)(B+1)K}
        \right).
    \]

    Comparing the second and the third term of the sum, we notice that $D > (B-1)(P + \sqrt{P})$ for $\varepsilon > 1$, so the second term is always larger than the third and we may bound the estimation error by
        $O\left(\frac{(B-1)^2\sqrt{3P \ln(2/\delta_\text{fail})} + KD^2}{(B-1)(B+1)K}\right).$
    The function $(x-1)^2/((x - 1)(x+1))$ is monotonically increasing for $x \geq 1$, so the choice $B = 2$ minimizes the expression above. Now, observe that the 
    first term is $O(\sqrt{\alpha/(K\tau)})$ and the second term is $\tilde{O}((\alpha K/\tau \varepsilon)^2)$, where the tilde notation suppresses the logarithmic dependence on $K$. To balance the estimation error, 
    we set these terms in relation to each other and solve for $K$. This shows that the asymptotic minimum 
    is achieved for $K = \varepsilon^{4/5}(\tau/\alpha)^{3/5}$.
    Using this value to bound the estimation error results in the bound stated in the theorem.
\end{proof}

\myparagraph{Comparing \texttt{RRMinHash} and \texttt{NoisyMinHash}} 
Comparing Thm.~\ref{thm:utility:rr:minhash} to Thm.~\ref{thm:utility:noisy:minhash}, both
analyses provide bounds on the absolute error in terms of the length $\tau$ of individual vectors, the neighboring notion $\alpha$, and the privacy budget $\varepsilon$.

Since the value $\alpha/\tau$ is between 0 and 1, the contribution of $(\alpha/\tau)^{4/5}$ to the error of \texttt{NoisyMinHash} (Theorem~\ref{thm:utility:noisy:minhash}) is smaller than the term $(\alpha/\tau)^{1/2}$ for \texttt{RRMinHash}. However, the $\varepsilon^{-1/2}$ dependence of \texttt{RRMinHash} is better than $\varepsilon^{-2/5}$ for $\varepsilon \geq 1$. 
This should mean that while \texttt{NoisyMinHash} might guarantee smaller error for small epsilon settings, the error decreases faster for \texttt{RRMinHash}.

In both mechanisms, the preprocessing time to generate a private vector is $O(K\tau)$ for a vector with $\tau$ set bits.
It consists of evaluating $K$
range-$B$ MinHash functions (each taking time $O(\tau)$) and sampling
$O(K)$ values from a uniform (\texttt{RRMinHash}) or Laplace distribution (\texttt{NoisyMinHash}). The similarity estimation of two vectors takes time $O(K)$. 
A private vector for \texttt{RRMinHash} consists of $K$ bits and the similarity estimation uses the Hamming distance,
while \texttt{NoisyMinHash} uses $K$ floating point values and uses Euclidean distance as basis for similarity estimation. 
Given the difficulty of correctly implementing the Laplace Mechanism~\cite{mironov2012significance}, \texttt{RRMinHash} has a simpler basis for a correct implementation. 

\mysection{Experimental Evaluation}
\label{sec:experiments}

All algorithms described in this paper where implemented in Python 3. 
The code, raw results, and evaluation notebooks can be accessed at \url{https://github.com/maumueller/ldp-jaccard}. 
Due to space restrictions, we only present a few selected results. 
See the Jupyter notebook at the web page for additional plots and tables.

\myparagraph{Experimental setting} 
We conduct experiments in two different directions. 

First, we create 
artificial vectors and test how well the algorithms estimate Jaccard similarity for a fixed privacy budget. 
We use the \emph{mean absolute error}
as our quality measure, which is defined as $\frac{1}{\ell}\sum_{i = 1}^{\ell} \vert d_i - e_i \vert$ for true similarities
$d_1, \ldots, d_\ell$ and their estimates $e_1, \ldots, e_\ell$ returned by the algorithm. In the experiment,
we create user vectors $x$ with $\tau \in \{20, 50, 100, 250, 500, 1000, 2000\}$ entries. 
For each such $x$, 
we create vectors $y'$ with $\tau$ entries and Jaccard similarity in $\{0.1, 0.5, 0.9\}$ to $x$. 
The number $K$ of hash functions considered is chosen from $\{10, 20, 30, \ldots, 500\}$. 
For each algorithm, we vary the privacy budget and internal parameters such as the range $B$ of the MinHash functions. 
All runs were repeated 100 times with random hash functions.

Second, we study how well these algorithms work on real-world datasets. 
Following~\cite{yanlocally}, we chose
the MovieLens and Last.FM dataset available at \url{https://grouplens.org/datasets/hetrec-2011}. 
We obtain a set representation by collecting all movies rated at least 4 (MovieLens, $m$ = 65\,536) and 
the top-20 artists (Last.FM, $m$ = 18\,739 ). 
The average set size is 178.1 ($\sigma = 187.5$) and 19.8 ($\sigma = 1.78$), respectively. 
To account for the influence of the size of the user vectors, we create different versions of these datasets. From the MovieLens dataset, we make three versions containing all users that have at least 50, 100, and 500 entries, respectively. This results in datasets with 1636, 1205, and 124 users.
 From the Last.FM dataset, we collect all users that have at least 20 entries which amounts to 1860 users. For each dataset, we take 50 query points at random among
all data points for which the 10-th nearest neighbor has at least Jaccard similarity 0.1. As quality measure, we use recall@$k$ (R@$k$) which measures the average number of times that the (index of the) true nearest neighbor is found among the first $k$ nearest neighbors in the private vectors. (Note that the true vectors are not revealed, so there cannot be a re-ranking step as is tradition in nearest neighbor search.) Moreover, we report on the approximate similarity ratio, which is defined as the ratio of the sum of similarities to the 10 original nearest neighbors, and the sum of similarities to the 10 nearest neighbors among the private vectors computed with their original similarities. 

For the whole evaluation, we will use range-2 MinHash (i.e., 1-bit MinHash~\cite{LiK10}) with $K \leq 100$ as a baseline for comparison. 
For all experiments, we set $\alpha = 1$, i.e. we allow for a single item change. 
Results for other values can be read off the plots by looking at different $\tau$ values. For example,
a combination $(\alpha = 1, \tau = 500)$ is identical to $(\alpha = 10, \tau = 50)$
since all bounds depend on the ratio of $\alpha$ and $\tau$. For all private mechanisms, we 
use $\delta = 0.0001$. 

\subsection{Result Discussion on Artificial Data}

\begin{figure}[t]
    \includegraphics[width=0.5\textwidth]{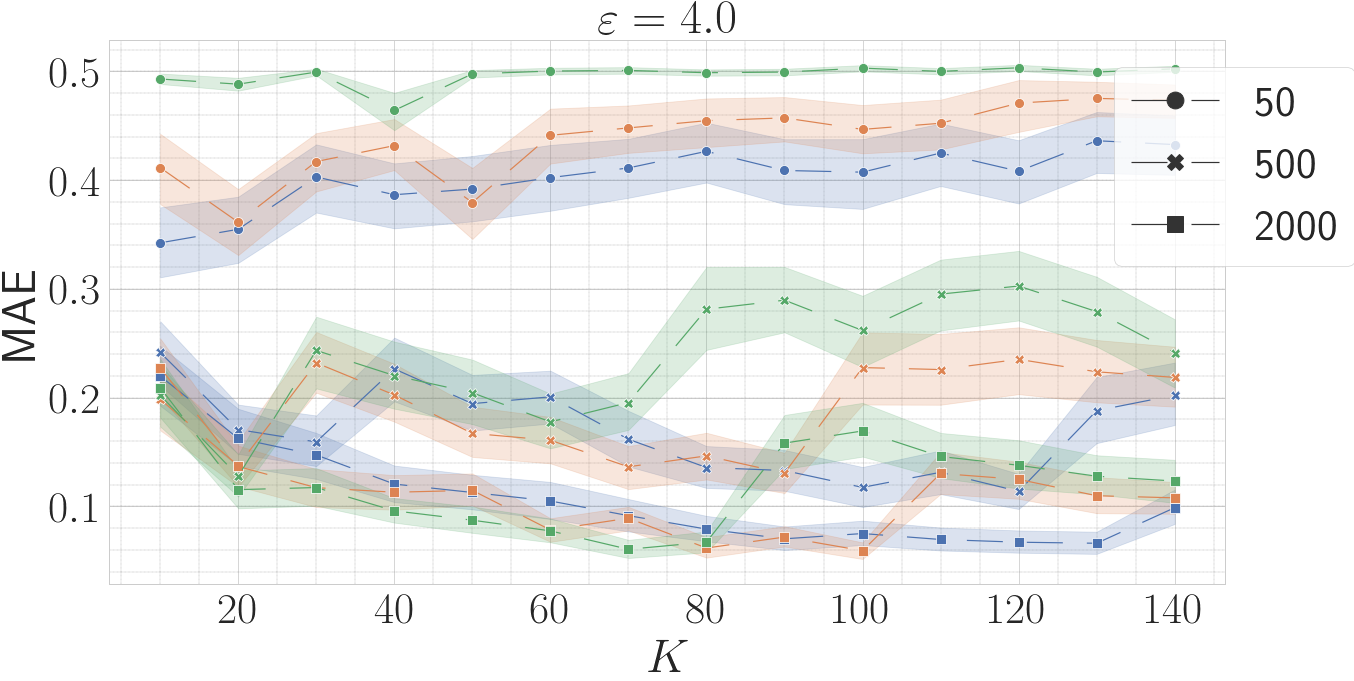}
    \includegraphics[width=0.5\textwidth]{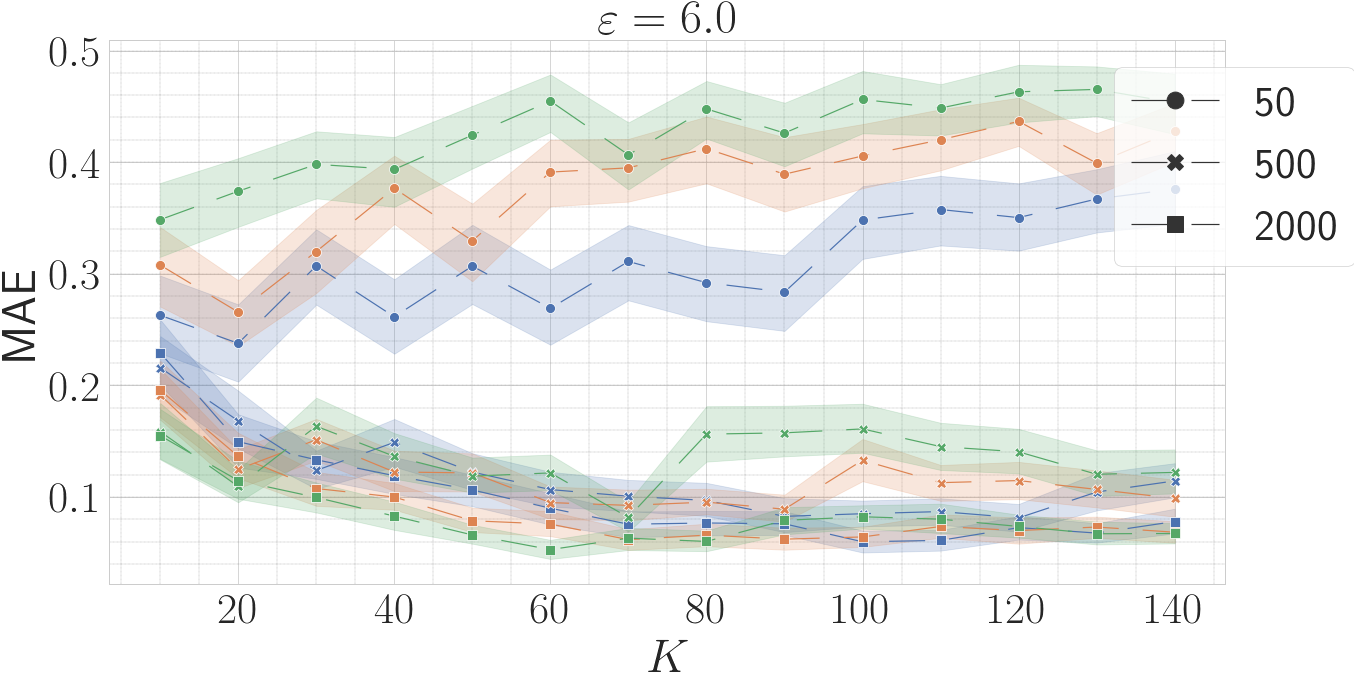}
    \includegraphics[width=0.5\textwidth]{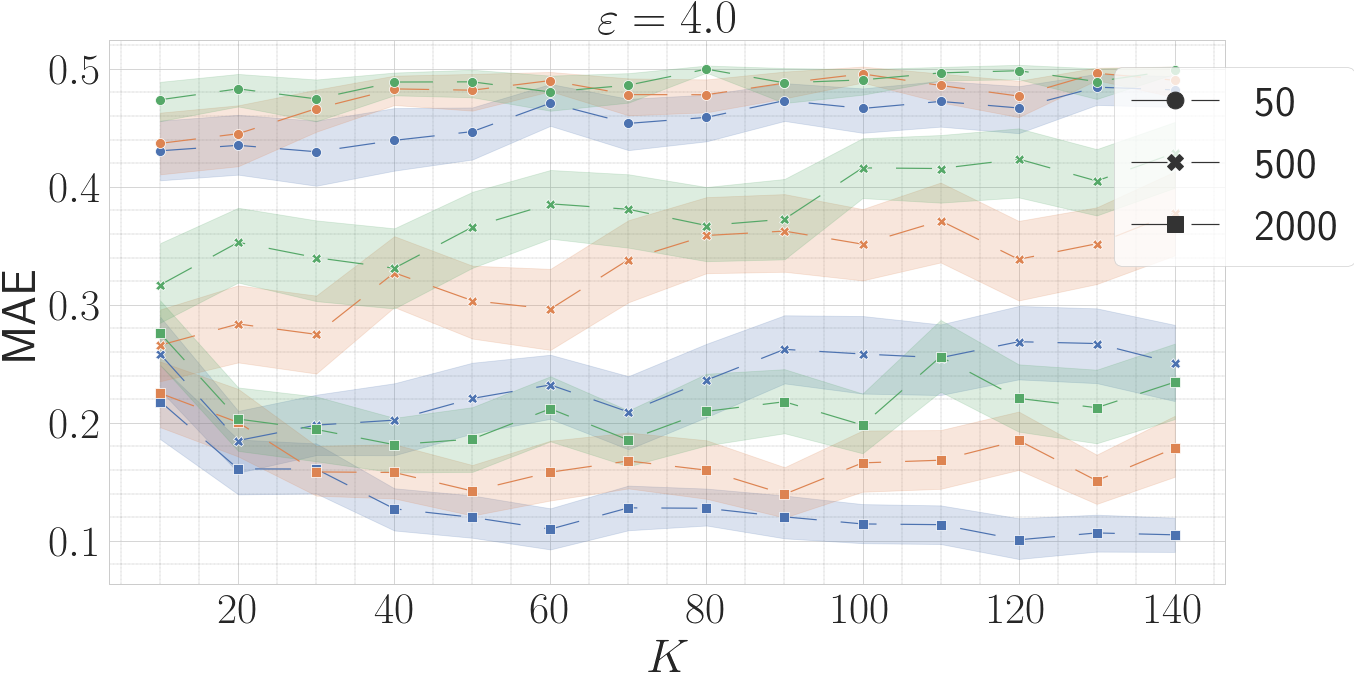}
    \includegraphics[width=0.5\textwidth]{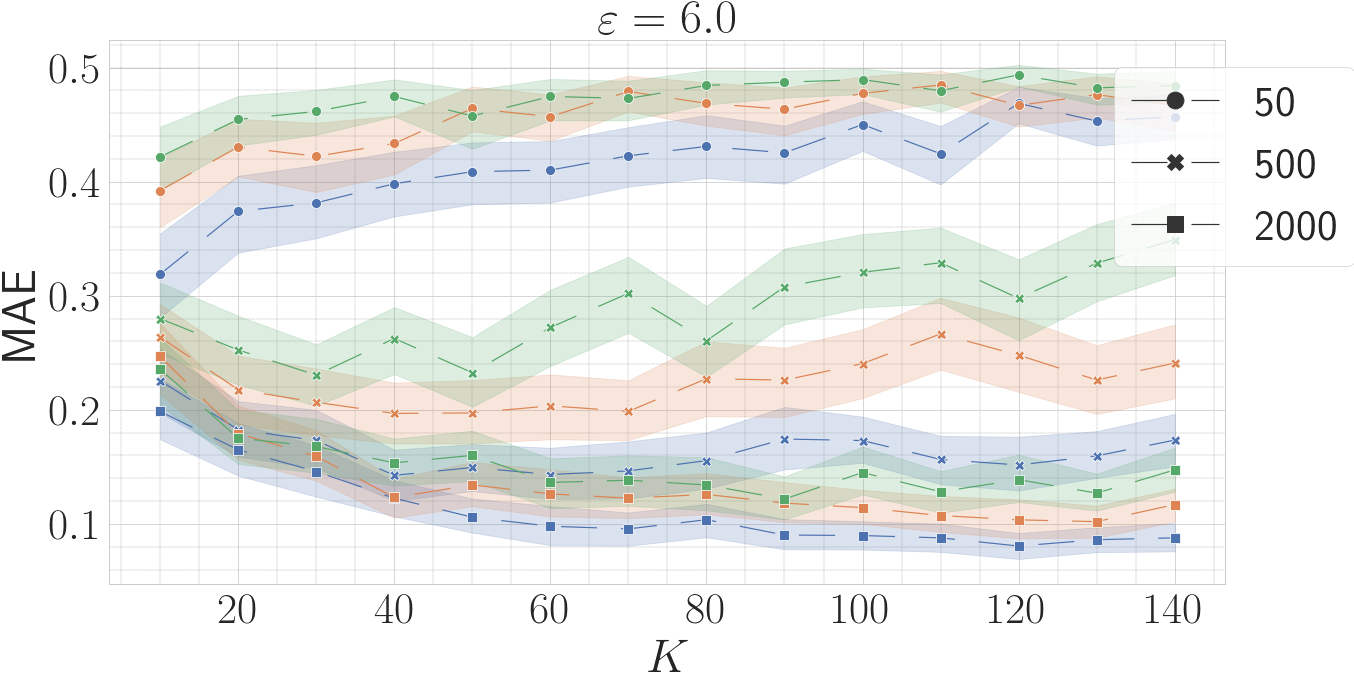}
    \caption{Results on synthetic vectors with $\tau \in \{50, 500, 2000\}$,  privacy budget
    $\varepsilon \in \{4, 6\}$, and vectors with Jaccard similarity of 0.5. Blue, red, green lines represent runs with choice of 2, 3, 5 for $B$, respectively; top: \texttt{RRMinHash}, bottom: \texttt{NoisyMinHash}.}
    \label{fig:artificial}
\end{figure}

Figure~\ref{fig:artificial} visualizes the mean absolute error (with standard deviation as error bars) for runs of \texttt{RRMinHash} (top) and \texttt{NoisyMinHash} (bottom) for privacy budgets of $\varepsilon = 4$ (left) and $\varepsilon = 6$ (right) and choices 2, 3, 5 of the $B$ parameter (blue/red/green lines). With respect to \texttt{RRMinHash} and a privacy budget of $\varepsilon = 4$, we notice that for each choice of $\tau$ the trend is that smaller $B$ values produce smaller absolute error, which is in accordance with our analysis in Section~\ref{sec:grr}. (Larger $B$ values can be found on the supplemental website; they performed much worse.) For vectors of 50 elements, the smallest MAE error is achieved with the smallest choice of $K$, resulting in an MAE of around 0.35. The error shrinks to around 0.15 for 500 elements (with $K$ of around 20), and 0.05 for vectors with 2\,000 elements (with $K$ around 80). 
The linear increase of $K$ with $\tau$ further motivates the choice of $K$ in Theorem~\ref{thm:utility:rr:minhash}. Increasing the privacy budget to $\varepsilon = 6$ further decreases 
the error but results in the same trends. We note that a growing privacy budget also 
corresponds to a larger $K$ choice, again as motivated in Theorem~\ref{thm:utility:rr:minhash}.
Increasing $K$ will sometimes result in worse error because of integer constraints in Lemma~\ref{lem:grr:on:budget}. From a practical point of view, one should choose $K$ as large as possible before this increase occurs. The trends are identical with regard to \texttt{NoisyMinHash}, but it is much clearer that a smaller choice of $B$ is preferable (as motivated in the proof of Theorem~\ref{thm:utility:noisy:minhash}). We achieve an MAE of around 0.43, 0.18, 0.1 
for vectors of size 50, 500, 2000 and $\varepsilon = 4$, respectively, slightly worse than \texttt{RRMinHash}.

\begin{figure}[t]
    \includegraphics[width=0.5\textwidth]{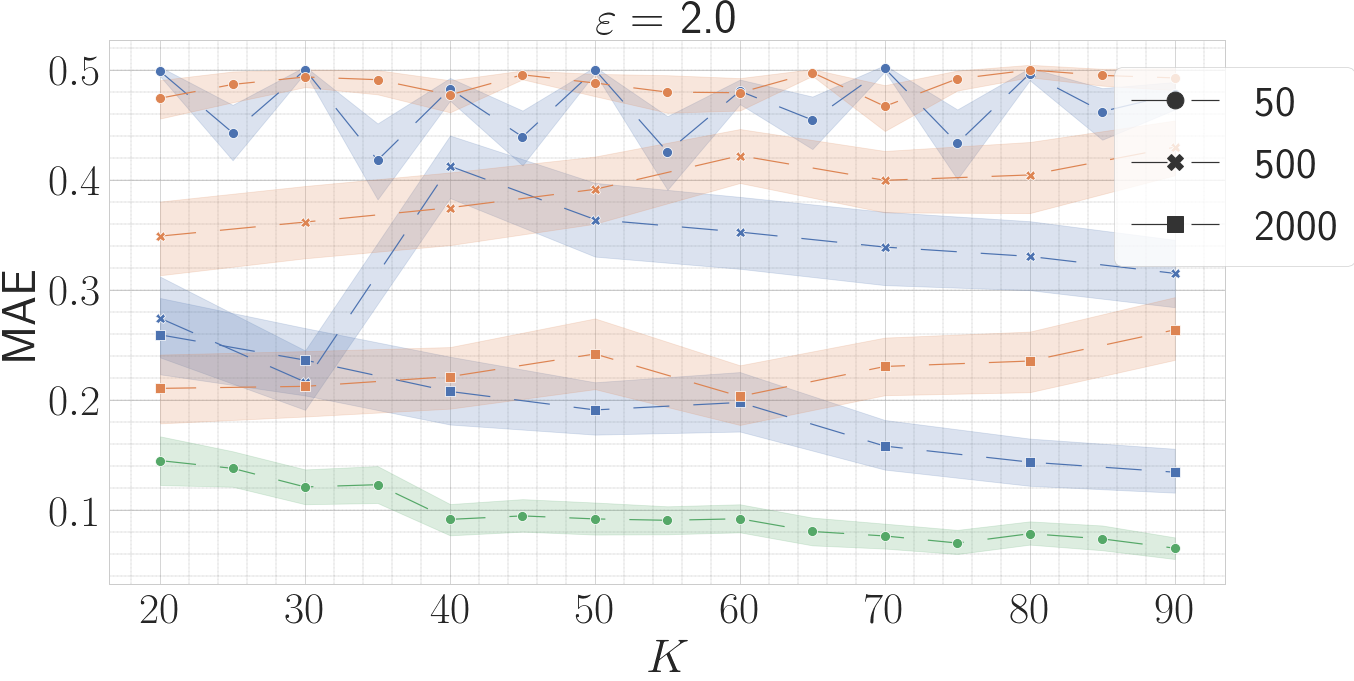}
    \includegraphics[width=0.5\textwidth]{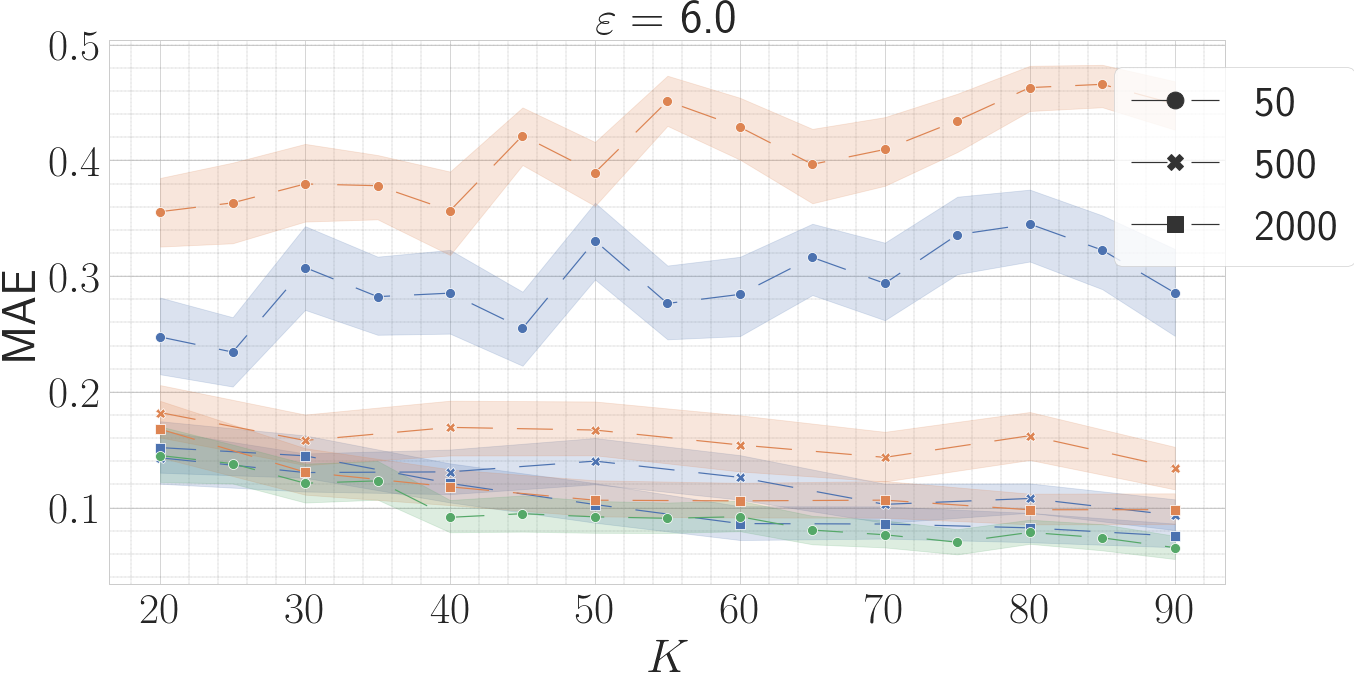}
    \caption{Results on synthetic vectors with $\tau \in \{50, 500, 2000\}$,  privacy budget
    $\varepsilon \in \{2, 6\}$, and vectors with Jaccard similarity of 0.5. Blue, red, green lines represent runs with \texttt{RRMinHash}, \texttt{NoisyMinHash}, \texttt{Range-2 MinHash} (non-private), respectively. There is only one line for MinHash because
    its error is independent of the vector size.}
    \label{fig:artificial:final}
\end{figure}
Figure~\ref{fig:artificial:final} sets our two mechanisms in relation to MinHash with $B=2$ and a privacy budget of $2$ (left) and $6$ (right). 
For $\varepsilon = 2$, we need large vectors to guarantee an error that is roughly a factor of two larger than that achieved by MinHash. For $\varepsilon = 6$, both larger vectors allow for an estimation vector that is nearly as small as MinHash. Again, \texttt{RRMinHash} achieves smaller error than \texttt{NoisyMinHash}, in particular for larger privacy budgets.

We conclude that \texttt{RRMinHash} with $B=2$ is a good choice in all considered experiments on artificial data. For small privacy budget, large user vectors are needed to get small estimation errors. A larger privacy budget allows to accommodate smaller vectors.

\subsection{Results on Real-World Data}

\begin{table}[t]
\begin{tabular}{l l c c c c}
Dataset & Algorithm & R@10 & R@50 & R@100 & Approx \\ \hline
Last.FM ($\tau = 20$) & MinHash & 0.42 & 0.72 & 0.82 & 0.55 \\
                      & RRMinHash & 0.04 / 0.16 & 0.15 / 0.38 & 0.25 / 0.51 & 0.19 / 0.35  \\ 
                      & NoisyMinHash & 0.03 / 0.06 & 0.11 / 0.19 &  0.19 / 0.31 &  0.16 / 0.23\\[.2em]
MovieLens ($\tau = 50$) & MinHash & 0.13 & 0.34 & 0.47 & 0.61 \\
                        & RRMinHash & 0.02 / 0.05 & 0.09 / 0.17 & 0.18 / 0.28 & 0.49 / 0.53 \\
                        & NoisyMinHash & 0.01 / 0.02 & 0.05 / 0.11 & 0.11 / 0.21 & 0.49 / 0.50 \\[.2em]
MovieLens ($\tau = 100$) & MinHash & 0.31 & 0.62 & 0.75 & 0.72 \\
                         & RRMinHash & 0.04 / 0.07 & 0.12 / 0.24 & 0.22 / 0.36 & 0.52 / 0.57 \\
                         & NoisyMinHash & 0.04 / 0.04 & 0.09 / 0.15 & 0.18 / 0.27 & 0.52 / 0.54 \\[.2em]
MovieLens ($\tau = 500$) & MinHash & 0.58 & 0.93 & 0.99 & 0.83 \\
                         & RRMinHash & 0.19 / 0.31 & 0.65 / 0.76 & 0.91 / 0.96 & 0.72 / 0.76 \\
                         & NoisyMinHash & 0.14 / 0.25 & 0.63 / 0.69 & 0.90 / 0.92 &  0.71 / 0.73 \\
\end{tabular}
\caption{Results on real-world datasets for different quality measures and privacy budget $\varepsilon$ of 4 and 8 (split up via ``/'' in individual cells).}
\label{table:real:world}
\end{table}

Table~\ref{table:real:world} summarizes the observed results for runs on the Last.FM and MovieLens datasets. Again, we set MinHash in relation to \texttt{RRMinHash} and \texttt{NoisyMinHash}. 
Motivated by the observations above we 
only discuss the case $B=2$.

We observe that \texttt{RRMinHash} achieves equal or better quality than \texttt{NoisyMinHash} in all measurements, so we focus the comparison on MinHash and \texttt{RRMinHash}. 
First, we note that the datasets are rather difficult. Even standard MinHash with $B=2$ does not 
achieve close to perfect recall, which means that all vectors are rather close to each other. 
The Last.FM dataset provides very small user vectors. Accordingly there is a big difference between the quality achieved by the two algorithms. For a privacy budget of $\varepsilon = 4$, 
the quality is between a factor of around 10 (R@10) and of around 3 (R@100, Approx) worse if solving the similarity search task on private vectors. For a privacy budget of $\varepsilon = 8$, these factors shrink to 1.5-3. With regard to MovieLens, we observe that it is difficult for MinHash to achieve high recall values for $\tau = 50$. Results for \texttt{RRMinHash} are again a factor 3-6 worse for privacy budget 4, with the exception of the relative approximation that is rather close (0.49 vs. 0.61). Quality increases slowly from 50 to 100 items, and rapidly for 500 items (because of its small size).

We summarize that there is a clear trade-off between the utility and privacy of the proposed mechanisms. The results on artificial and real-world data show that to ensure good utility under a small privacy budget, user vectors have to contain many items, say in the 100s. Many of the theoretical choices translated well into practice. Most interestingly, while the upper bounds in the theory section painted an unclear picture about the utility at a fixed privacy budget, our empirical analysis clearly suggests that \texttt{RRMinHash} is both easier to implement and achieves higher utility for the same privacy budget.

\mysection{Related Work}
The paper by Kenthapadi~\cite{Kenthapadi13} shows how to
estimate vector differences under the $\ell_2$ norm in a differentially private setting in the centralized model of differential privacy.
More precisely, their algorithm has privacy guarantees with respect to a single
element change (i.e., one user changes one item). In very recent work, Dhaliwal et
al.~\cite{Dhaliwal19} show how to achieve the same guarantees when the
privacy-guarantees are over the change of a fraction of a user vector in the central model. Both approaches 
apply a Johnson-Lindenstrauss transform~\cite{JohnsonL84} and add noise of a certain
scale to the resulting matrix. Our \texttt{NoisyMinHash} approach can be seen as a natural generalization of their method, but there are some stand-alone features such as the mapping to $B$ buckets.

With respect to similarity estimation under Jaccard similarity, 
the paper by Riazi et al.~\cite{Riazi16} describes a privacy-preserving approach
for similarity estimation both for inner product similarity (using
SimHash~\cite{Charikar02}) and Jaccard similarity (using MinHash). Their 
privacy notion does not satisfy differential privacy.

The paper by Yan
et al.~\cite{yanlocally} is closest to our approach. 
It discusses an LDP approach based on MinHash by 
selecting certain hash values in a differentially private manner
using the exponential mechanism. As we argue in Appendix~\ref{app:yan},
their approach does not provide the guarantees they state and quickly
degrades to a basic MinHash approach without noise addition.

Concurrent to our work, Pagh and Stausholm~\cite{pagh2020efficient} describe 
LDP sketches for approximating the number of items in a set. Their sketches are \emph{linear},
which allows them to approximate the size of the union and the intersection of two sets, and thus their Jaccard similarity. In contrast to our bounds, their bounds rely on the universe size of set elements. It would be interesting to compare their mechanism to ours in a practical setting, in particular because their lower-order error terms~\cite[Theorem 1]{pagh2020efficient} suggest that they need much larger vectors than the ones considered in our empirical study in Section~\ref{sec:experiments}.

Finally, this paper studied the privacy/utility-tradeoff achievable with our proposed methods. While an important issue, it does not discuss (un)desirable privacy budgets, which will be application-specific and lack consensus~\cite{Dwork_Kohli_Mulligan_2019}.

\vspace{-1em}
\bibliographystyle{splncs03}
\bibliography{lit}

\appendix

\section{Review of ``Locally Private Jaccard Similarity estimation'' by Yan et al.}
\label{app:yan}
In this section, we will consider the approach of Yan et al.~\cite{yanlocally}. We will first describe their approach, and then 
discuss short-comings of their analysis. 

\subsection{Similarity Estimation in~\cite{yanlocally}}

The approach of Yan et al. in~\cite{yanlocally} works as follows. 
First, choose $K$ MinHash functions $h_1, \ldots, h_K$ and distribute them to the users. Let $x_u$ be the user vector containing $N$ elements. For each $i \in \{1, \ldots, K\}$, do the following: (i) apply  hash function $h_i$ to the user vector,  keep track of the order of the $N$ elements under $h_i$, and (ii)
choose an element using the Exponential mechanism~\cite{dwork2014algorithmic} with utility function that gives utility $N - 1$ to the first (smallest) element, and goes down to utility $0$ for the last element under the random order. The response of the user is the $K$ elements chosen in this way.

Given two responses $\hat{x}$ and $\hat{y}$, \cite{yanlocally} returns the value $\vert\hat{x} \cap \hat{y}\vert / K$.  

\subsection{Criticism}

We will focus on the following main issues:

\begin{enumerate}
    \item Their profile perturbation using the expontial mechanism is not $(\varepsilon, 0)$-differentially private.
    \item Their estimation algorithm is not an unbiased estimator of the Jaccard similarity.
    \item Their self-adaption mechanism quickly degrades to pure MinHash.  
\end{enumerate}

With regard to the first point, we note that each user applies the exponential mechanism to choose among their set of values, i.e., the actual hash values obtained for their vector. All values different from those have zero probability of being chosen. This, however, cannot be differential private, quoting~\cite[Page 38]{dwork2014algorithmic}:

\begin{quote}
    It is important that the range of \emph{potential} prices [values] is independent of the actual bids [the hash values which were observed for the user]. Otherwise there would exist a price [value] with non-zero weight in one dataset and zero weight in a neighboring dataset, violating differential privacy. 
\end{quote}

For their approach, this means that if two user vectors differ in exactly one element, there is a non-zero probability that this value is picked for one vector, and a zero probability in the other vector (because it is not present). This violates differential privacy.

With regard to the second point, just returning the number of collisions (over $K$) in the perturbed sketches is not an unbiased estimator. 
It is well-known that if $x$ and $y$ have Jaccard similarity $J(x,y)$, the collision probability under a single minhash value is exactly $J(x,y)$. Repeating the process $K$ times, let $d$ be the number of hash collisions. The unbiased estimator is then $d/K$, with an error of $\sqrt{1/K}$. This is the value that they were to report in their Definition~2. However, this process does not generalize if they pick values using the Exponential mechanism. For example, if one vector reports the second-smallest hash value, and the other reports the smallest hash value, the collision probabiltiy is not $J(x,y)$. It is easy to see that their estimate does not reflect the Jaccard similarity in their plots, e.g., Figure~3 (A)--(C). The baseline does not recover the distance, independent of the epsilon value. 

The question does remains why they get good results for their other approaches in their Figure~3. This brings us to 3.: their so-called ``self-adaptation''. Using self-adaption, they restrict the space of choices for the Exponential mechanism even further. Given the parameter choices they 
mention at~\cite[Page 8]{yanlocally}, for $\varepsilon = 1$, the exponential mechanism chooses only between the smallest and second-smallest element, for $\varepsilon \geq 1.2$ it chooses only from the set containing only the smallest hash value. Since this defaults to just using MinHash, it is not surprising that they measure a small estimation error. However, the sketch is just the MinHash sketch, and since all hash functions are shared, this clearly violates differential privacy: the mechanism is \emph{deterministic} at this point. 

\section{A Useful Utility Bound for the Laplace Mechanism}
\label{app:laplace:utility}

For completeness, we reproduce Theorem~3.8 in~\cite{dwork2014algorithmic} with the notation of our setup. 

\begin{theorem}
Let $f\colon \{0,1\}^m \rightarrow [B]^K$, and let $y = f(x) + (Y_1, \ldots, Y_K)$ with $Y_i \sim \textnormal{Lap}(\Delta/\varepsilon)$. Then for all $\delta \in (0, 1]$:
\[
    \Pr\left[\|f(x) - y\|_\infty \geq \ln (K/\delta) \frac{\Delta}{\varepsilon}\right] \leq \delta.
\]
\end{theorem}

\section{A Running Example}
\label{app:example}

\begin{figure}[t]
    \includegraphics[width=\textwidth]{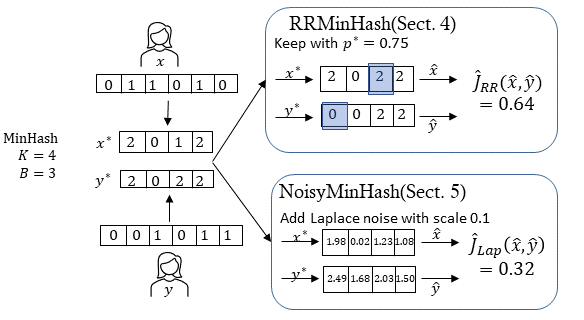}
    \caption{Overview of the algorithms presented in this work. Two users whose vectors have Jaccard similarity 0.5 apply 4 range-3 MinHash functions to their vectors of length $m=6$. In the top row, each user uses
    randomize response with flip probability $1/4$ to generate their private vectors. 
    In the bottom row, they use noise drawn from a Laplace distribution. The result 
    of the estimation is shown to the right. }
    \label{fig:scheme}
\end{figure}

Figure~\ref{fig:scheme} provides a running example of our setup and the different mechanisms. We consider two users that can have items in the set $\{1, \ldots, 6\}$. User $x$ has items $\{2,3,5\}$, user $y$ has items $\{3, 5, 6\}$. Their Jaccard similarity is thus $2/4=0.5$. Users have access to $K=4$ MinHash functions that map to $B = 3$ buckets $\{0,1,2\}$. 
Applying these functions to user $x$ gives the sketch $x^\ast = (2, 0, 1, 2)$, applying them to $y$'s user results in $y^\ast = (2, 0, 2, 2)$. 
These two vectors are made private using either Generalized Randomized Response or addition of Laplace noise. 

In the first case, consider a probability of $p^\ast = 3/4$, which means that each user randomizes their answer with probability $1/4$. For $x^\ast$, the third element is flipped, and the random choice from $\{0,1,2\} \setminus \{1\}$ results in the answer 2. The private vector $\hat{x}$ is thus $(2,0,2,2)$. For $y^\ast$, the first element is flipped and the random choice among $\{0,1,2\} \setminus \{2\}$ results in 0. The answer $\hat{y}$ is thus $(0,0,2,2)$. We count 2 collisions between $\hat{x}$ and $\hat{y}$, and plugging this into the estimation formula~\eqref{eq:grr:estimate} gives an estimate of 0.64.

In the second case, consider that Laplace noise is added with scale 0.1.
Each user independently adds noise by drawing 4 samples from Lap(0.1) and adding the samples to the original vector. This results in the two vectors $\hat{x} = (1.98, 0.02, 1.23, 1.08)$ and $\hat{y} = (2.49, 1.68, 2.03, 1.50)$ (rounded to two digits). Computing the squared Euclidean distance between these vectors and adjusting for the parameters as stated in~\eqref{eq:estimate} results in an estimate of 0.32 for the Jaccard similarity of $x$ and $y$.

\end{document}